\documentclass{article} 
\usepackage{nips07submit_e,times}
\usepackage{amsmath,amssymb,amsthm}
\usepackage{algorithm,algorithmic}
\usepackage{graphicx}
\usepackage{multirow}

\usepackage{amsmath,amsthm,amssymb,units,color,graphicx}
\newtheorem{theorem}{Theorem}
\newtheorem{lemma}[theorem]{Lemma}

\newtheorem{corollary}[theorem]{Corollary}

\def\R{\mathbb{R}}

\def\eps{\varepsilon}

\newcommand{\Rg}[2]{R_{{#1},{#2}}}
\newcommand{\rg}[2]{r_{{#1},{#2}}}
\newcommand{\rhog}[2]{\rho_{{#1},{#2}}}
\newcommand{\ls}[2]{\ell_{{#1},{#2}}}
\newcommand{\pr}[2]{p_{{#1},{#2}}}
\newcommand{\p}[1]{p_{{#1}}}
\newcommand{\eplus}{E_{t}}
\newcommand{\eplustwo}{E_{t+1}}
\newcommand{\epp}{E_{t,t+1}}
\newcommand{\sone}{S_1}
\newcommand{\stwo}{S_2}
\newcommand{\changeto}[2]{#2}

\title{A Parameter-free Hedging Algorithm}
\author{Kamalika Chaudhuri\\
ITA, UC San Diego\\
{\tt kamalika@soe.ucsd.edu}
\And Yoav Freund\\
CSE, UC San Diego\\
{\tt yfreund@ucsd.edu}
\And Daniel Hsu\\
CSE, UC San Diego\\
{\tt djhsu@cs.ucsd.edu}}

\begin{document}

\maketitle

\begin{abstract}
We study the problem of decision-theoretic online learning (DTOL).
Motivated by practical applications, we focus on DTOL when the number of
actions is very large. Previous algorithms for learning in this framework
have a tunable learning rate parameter, and a barrier to using
online-learning in practical applications is that it is not understood how
to set this parameter optimally, particularly when the number of actions is
large.

In this paper, we offer a clean solution by proposing a novel and
completely parameter-free algorithm for DTOL.  We introduce a
new notion of regret, which is more natural for applications with a large
number of actions. We show that our algorithm achieves good performance
with respect to this new notion of regret; in addition, it also achieves
performance close to that of the best bounds achieved by previous
algorithms with optimally-tuned parameters, according to previous notions
of regret.

\end{abstract}

\section{Introduction}

In this paper, we consider the problem of decision-theoretic online
learning (DTOL), proposed by Freund and Schapire~\cite{FS97}. DTOL is a
variant of the problem of prediction with expert advice~\cite{LW94,
Vov98}. In this problem, a learner must assign probabilities to a fixed set of actions in a sequence of rounds.
After each assignment, each action incurs a loss (a value in $[0,1]$); the
learner incurs a loss equal to the expected loss of actions for that round,
where the expectation is computed according to the learner's current
probability assignment.
The \emph{regret} (of the learner) to an action is the difference between
the learner's cumulative loss and the cumulative loss of that action.
The goal of the learner is to achieve, on any sequence of losses, low
regret to the action with the lowest cumulative loss (the best action).

DTOL is a general framework that captures many learning problems of
interest. For example, consider tracking the hidden state of an object in a
continuous state space from
noisy observations~\cite{CFH09}. To look at tracking in a DTOL framework,
we set each action to be a path (sequence of states) over the state space. The loss of an action at time $t$ is the distance between the observation at time $t$ and
the state of the action at time $t$, and the goal of the learner is to
predict a path which has loss close to that of the action with the lowest cumulative loss.

The most popular solution to the DTOL problem is the Hedge
algorithm~\cite{FS97, FS99}. In Hedge, each action is assigned
a probability, which depends on the cumulative loss of this action and a
parameter $\eta$, also called the {\em learning rate}. By appropriately
setting the learning rate as a function of the iteration~\cite{ACBG02,
CBMS07} and the number of actions, Hedge can achieve a regret upper-bounded
by $O(\sqrt{T \ln N})$, for each iteration $T$, where $N$ is the
number of actions.  
This bound on the regret is optimal as there is a $\Omega(\sqrt{T \ln N})$ lower-bound~\cite{FS99}.

In this paper, motivated by practical applications such as tracking, we
consider DTOL in the regime where \changeto{$N$, the number of
actions,}{the number of actions $N$} is very large.
A major barrier to using
online-learning for practical problems is that when $N$ is large, it is not understood how to set the learning rate $\eta$. \cite{CBMS07, ACBG02} suggest setting $\eta$ as a fixed function of the number of actions $N$.
However, this can lead to poor performance, as we illustrate by an example in Section~\ref{sec:expts}, and the degradation in performance is particularly exacerbated as $N$ grows larger. One way to address this is by
simultaneously running multiple copies of Hedge with multiple values of the
learning rate, and choosing the output of the copy that performs the best in an online way. However, this solution is impractical for
real applications, particularly as $N$ is already very large. (For more details about these solutions, please see Section~\ref{sec:related}.)

In this paper, we take a step towards making online learning more practical  by proposing a novel, completely adaptive algorithm for DTOL. Our algorithm is called NormalHedge. NormalHedge is very simple  and easy to implement, and in each round, it simply involves a single line search, followed by an updating of weights for all actions.

A second issue with using online-learning in problems such as tracking, 
where $N$ is very large, is that the regret to the {\em best action} is not an effective measure of performance. 
For problems such as tracking, one expects to have a lot of actions that are
close to the action with the lowest loss. As these actions also have low
loss, measuring performance with respect to a small group of actions that
perform well is extremely reasonable -- see, for example,
Figure~\ref{fig:quantile}.

In this paper, we address this issue by introducing a new notion of regret, which is more natural
for practical applications. We order the cumulative losses of all actions from lowest to highest and define the {\em regret of the learner to the top $\epsilon$-quantile} to be the difference
between the cumulative loss of the learner and the $\lfloor
\epsilon N \rfloor$-th element in the sorted list.

\begin{figure}
\begin{center}
\includegraphics[width=0.4\textwidth]{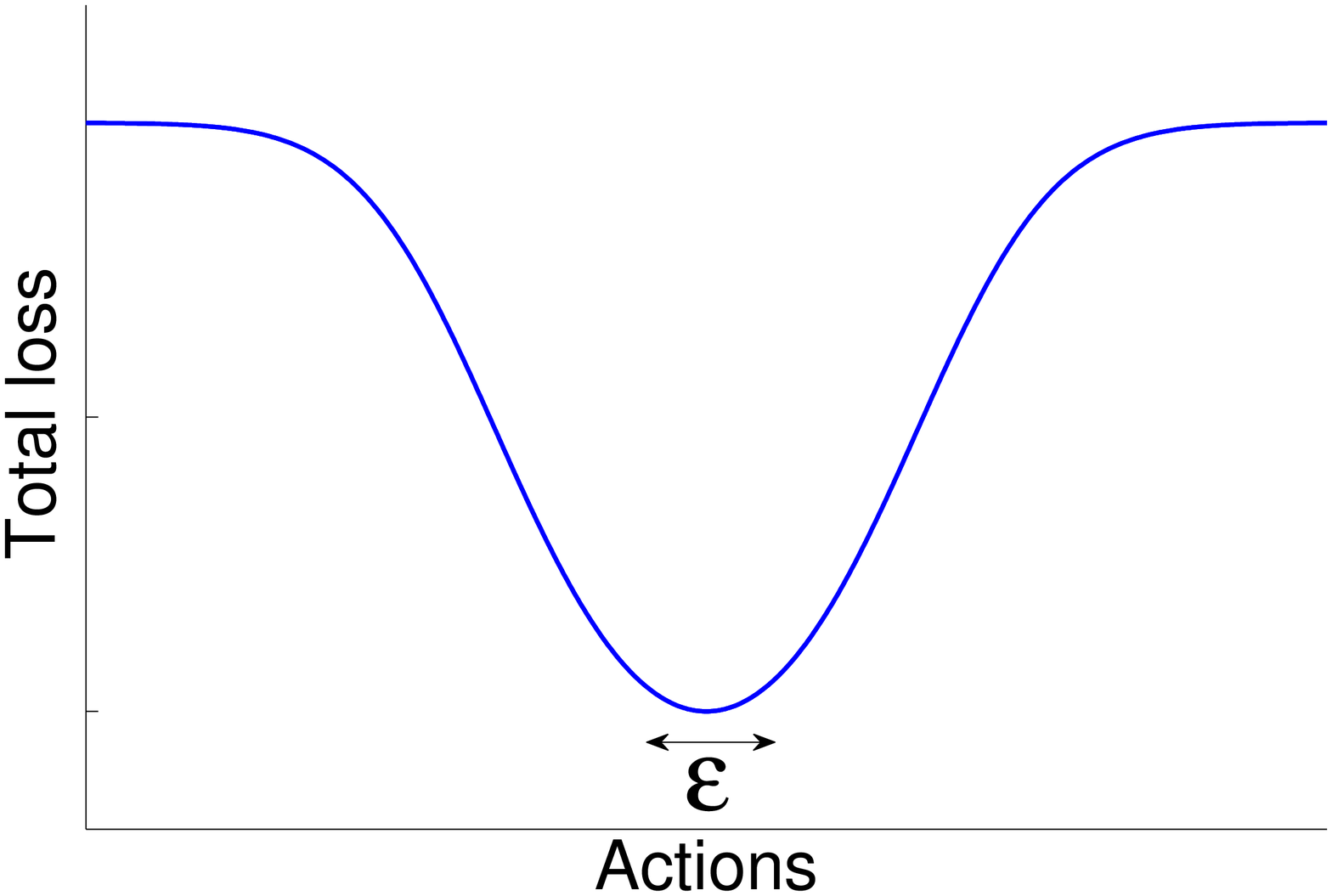}
\end{center}
\caption{A new notion of regret.
Suppose each action is a point on a line, and the total losses are as given
in the plot.
The regret to the top $\epsilon$-quantile is the difference between the
learner's total loss and the total loss of the worst action in the indicated
interval of measure $\epsilon$.
}
\label{fig:quantile}
\end{figure}

We prove that for NormalHedge, the regret to the top
$\epsilon$-quantile of actions is at most
\[
O\left(\sqrt{T \ln \frac{1}{\epsilon}} + \ln^2 N\right),
\]
which holds {\em simultaneously for all $T$ and $\epsilon$}. If we set
$\epsilon = 1/N$, we get that the regret to the best action is
upper-bounded by $O\left(\sqrt{T \ln N} + \ln^2 N\right)$, which is only
slightly worse than the bound achieved by Hedge with optimally-tuned
parameters. Notice that in our regret bound, the term involving $T$ has no
dependence on $N$. In contrast, Hedge cannot achieve a regret-bound of this
nature uniformly for all $\epsilon$.  (For details on how Hedge can be
modified to perform with our new notion of regret, see Section~\ref{sec:related}).

NormalHedge works by assigning each action $i$ a potential; actions which have lower cumulative loss than the algorithm are assigned a potential
$\exp(\Rg{i}{t}^2/2 c_t)$, where $\Rg{i}{t}$ is the regret of action $i$
and $c_t$ is an adaptive scale parameter, which is adjusted from one round
to the next, depending on the loss-sequences.
Actions which have higher cumulative loss than the algorithm
are assigned potential $1$. The weight assigned to an action in each round is then 
proportional to the derivative of its potential. One can also interpret Hedge as a potential-based algorithm, and under this interpretation, the potential assigned by Hedge to action $i$ is proportional to $\exp(\eta
\Rg{i}{t})$.
\changeto{This potential used by Hedge and other
related algorithms differs significantly from the one we use.}{This
potential used by Hedge differs significantly from the one we use.}
Although other potential-based methods have been
considered in the context of online learning~\cite{CBL03}, our potential function is very novel, and to the best of our
knowledge, has not been studied
in prior work. Our proof techniques are also different from previous
potential-based methods. 
Another useful property of NormalHedge, which Hedge does not possess, is that it assigns zero weight to any action whose cumulative loss is larger than
the cumulative loss of the algorithm itself. In other words, non-zero
weights are assigned only to actions which perform better than the
algorithm. In most applications\changeto{ of DTOL}{,} we expect a small set of the actions
to perform significantly better than most of the actions. \changeto{As the regret
of the hedging algorithm is guaranteed to be small, this means that the
algorithm will perform better than most of the actions and will 
therefore assign them zero probability.}{The regret of the algorithm is
guaranteed to be small, which means that the algorithm will perform better
than most of the actions and thus assign them zero probability.}

\cite{CL06,HK08} have proposed more recent solutions to DTOL in which the regret of Hedge to the best action is upper bounded by a function of $L$, the
loss of the best action, or by a function of the variations in the losses. These bounds
can be sharper than the bounds with respect to $T$. Our analysis (and in fact, to our knowledge, any analysis based on potential functions in the style
of~\cite{Gen03,CBL03}) do not directly yield these kinds of bounds. We
therefore leave open the question of finding an adaptive algorithm for DTOL which has regret upper-bounded by a function that depends on the loss 
of the best action.

The rest of the paper is organized as follows. In Section 2, we provide
NormalHedge. In Section 3, we provide an example that illustrates the
suboptimality of standard online learning algorithms, when the parameter is
not set properly. In Section 4, we discuss Related Work. In Section 5, we
present some outlines of the proof. The proof details are in the
Supplementary Materials.

\section{Algorithm} \label{sec:alg}

\subsection{Setting}

We consider the decision-theoretic framework for online learning.
In this setting, the learner is given access to a set of $N$ actions, where
$N \geq 2$. In round $t$, the learner chooses a weight distribution $\p{t} =
(\pr{1}{t}, \ldots, \pr{N}{t})$ over the actions $1, 2, \ldots, N$.
Each action $i$ incurs a loss $\ls{i}{t}$, and the learner incurs the
expected loss under this distribution:
$$ \ls{A}{t} = \sum_{i=1}^N \pr{i}{t}\ls{i}{t}. $$
The learner's instantaneous regret to an action $i$ in round $t$ is
$\rg{i}{t} = \ls{A}{t} - \ls{i}{t}$,
and its (cumulative) regret to an action $i$ in the first $t$
rounds is $$ \Rg{i}{t} = \sum_{\tau=1}^t \rg{i}{\tau}. $$
We assume that the losses $\ls{i}{t}$ lie in an interval of length $1$
(\emph{e.g.}~$[0,1]$ or $[-1/2,1/2]$; the sign of the loss does not
matter).
The goal of the learner is to minimize this cumulative regret $\Rg{i}{t}$
to any action $i$ (in particular, the best action), for any value of $t$.

\subsection{Normal-Hedge}

\begin{figure}[t]
\begin{center}
\framebox[\textwidth]{\begin{minipage}{0.9\textwidth}
\vskip 0.1in
Initially:  Set $\Rg{i}{0} = 0$, $\pr{i}{1} = 1/N$ for
each $i$. \\
For $t = 1, 2, \ldots$
\begin{enumerate}
\item Each action $i$ incurs loss $\ls{i}{t}$.
\item Learner incurs loss $\ls{A}{t} = \sum_{i=1}^N \pr{i}{t}
\ls{i}{t}$.
\item Update cumulative regrets:
$\Rg{i}{t} = \Rg{i}{t-1} + (\ls{A}{t} - \ls{i}{t})$ for each $i$.
\item Find $c_t > 0$ satisfying
$\frac{1}{N} \sum_{i=1}^N \exp\left(\frac{([\Rg{i}{t}]_+)^2}{2c_t}\right)
= e$.
\item Update distribution for round $t+1$:
$\pr{i}{t+1} \propto \frac{[\Rg{i}{t}]_+}{c_t}
\exp\left(\frac{([\Rg{i}{t}]_+)^2}{2c_t}\right)$
for each $i$.
\end{enumerate}
\end{minipage}}
\caption{The Normal-Hedge algorithm.}
\label{fig:alg}
\end{center}
\vskip -0.2in
\end{figure}

Our algorithm, Normal-Hedge, is based on a potential function reminiscent
of the half-normal distribution, specifically
\begin{equation} \label{eqn:potential}
\phi(x,c) \ = \ \exp\left(\frac{([x]_+)^2}{2c}\right)
\quad \text{for $x \in \R, c > 0$}
\end{equation}
where $[x]_+$ denotes $\max\{0,x\}$.
It is easy to check that this function is separately convex in $x$ and $c$,
differentiable, and twice-differentiable except at $x = 0$.

In addition to tracking the cumulative regrets $\Rg{i}{t}$ to each action
$i$ after each round $t$, the algorithm also maintains a scale parameter
$c_t$. This is chosen so that the average of the potential, over all
actions $i$, evaluated at
$\Rg{i}{t}$ and $c_t$, remains constant at $e$:
\begin{equation} \label{eq:potential}
\frac1N \sum_{i=1}^N \exp\left(\frac{([\Rg{i}{t}]_+)^2}{2c_t}\right)
\ = \ e.
\end{equation}
We observe that since $\phi(x,c)$ is convex in $c > 0$, we can determine
$c_t$ with a line search.

The weight assigned to $i$ in round $t$ is set proportional to the
first-derivative of the potential, evaluated at $\Rg{i}{t-1}$ and
$c_{t-1}$:
\begin{eqnarray*}
\pr{i}{t}
& \propto & \left. \frac{\partial}{\partial x} \phi(x,c)
\right \vert_{x=\Rg{i}{t-1},c=c_{t-1}}
\ = \
\frac{[\Rg{i}{t-1}]_+}{c_{t-1}} \exp\left(
\frac{([\Rg{i}{t-1}]_+)^2}{2c_{t-1}}
\right).
\end{eqnarray*}
Notice that the actions for which $\Rg{i}{t-1} \leq 0$ receive zero weight in round
$t$.

We summarize the learning algorithm in Figure~\ref{fig:alg}.

\section{An Illustrative Example} \label{sec:expts}

In this section, we present an example to illustrate that setting the
parameters of DTOL algorithms as a function of $N$, the
total number of actions, is suboptimal. To do this, we compare the
performance of NormalHedge with two representative algorithms:
a version of Hedge due to~\cite{CBMS07}, and the Polynomial Weights
algorithm, due to~\cite{GLS01,Gen03}.
Our experiments with this example indicate that the performance of both
these algorithms suffer because of the suboptimal setting of the
parameters; on the other hand, NormalHedge automatically adapts to the
loss-sequences of the actions.

The main feature of our example is that the effective number of actions $n$
(\emph{i.e.}~the number of distinct actions) is smaller than the total
number of actions $N$. 
Notice that without prior knowledge of the actions and their
loss-sequences, one cannot determine the effective number actions in
advance; as a result, there is no direct method by which Hedge and
Polynomial Weights could set their parameters as a function of $n$. 

Our example attempts to model a practical scenario where one often finds multiple actions with loss-sequences which are almost
identical. For example, in the tracking problem, groups of paths which are
very close together in the state space, will have very close
loss-sequences. Our example indicates that in this case, the performance of
Hedge and the Polynomial Weights will depend on the discretization of the
state space, however, NormalHedge will comparatively unaffected by such discretization.

Our example has four parameters: $N$, the total number of actions; $n$, the
effective number of actions (the number of distinct actions); $k$, the
(effective) number of good actions;
and $\epsilon$, which indicates how much better the good actions are
compared to the rest.
Finally, $T$ is the number of rounds.

The instantaneous losses of the $N$ actions are represented by a $N \times
T$ matrix $B_N^{\varepsilon,k}$; the loss of action $i$ in round $t$ is the
$(i,t)$-th entry in the matrix.
The construction of the matrix is as follows.
First, we construct a (preliminary) $n \times T$ matrix $A_n$ based
on the $2^d \times 2^d$ Hadamard matrix, where $n = 2^{d+1}-2$.
This matrix $A_n$ is obtained from the $2^d \times 2^d$ Hadamard matrix by
(1) deleting the constant row, (2) stacking the remaining rows on top of
their negations, (3) repeating each row horizontally $T/2^d$ times, and
finally, (4) halving the first column.
We show $A_6$ for concreteness:
\[
A_6 = \left[
\begin{array}{rrrr|rrrr|rrrr|r}
-\nicefrac12 & +1 & -1 & +1 & -1 & +1 & -1 & +1 & -1 & +1 & -1 & +1 & \ldots \\
-\nicefrac12 & -1 & +1 & +1 & -1 & -1 & +1 & +1 & -1 & -1 & +1 & +1 & \ldots \\
-\nicefrac12 & +1 & +1 & -1 & -1 & +1 & +1 & -1 & -1 & +1 & +1 & -1 & \ldots \\
+\nicefrac12 & -1 & +1 & -1 & +1 & -1 & +1 & -1 & +1 & -1 & +1 & -1 & \ldots \\
+\nicefrac12 & +1 & -1 & -1 & +1 & +1 & -1 & -1 & +1 & +1 & -1 & -1 & \ldots \\
+\nicefrac12 & -1 & -1 & +1 & +1 & -1 & -1 & +1 & +1 & -1 & -1 & +1 & \ldots \\
\end{array}
\right]
\]
If the rows of $A_n$ give the losses for $n$ actions over time, then it is
clear that on average, no action is better than any other.
Therefore for large enough $T$, for these losses, a typical algorithm will
eventually assign all actions the same weight.
Now, let $A_n^{\varepsilon,k}$ be the same as $A_n$ except that
$\varepsilon$ is subtracted from each entry of the first $k$ rows,
\emph{e.g.}
\[
A_6^{\varepsilon,2} = \left[
\begin{array}{rccc|cccc|r}
-\nicefrac12 - \varepsilon & +1 - \varepsilon & -1 - \varepsilon & +1 - \varepsilon & -1
- \varepsilon& +1 - \varepsilon & -1 - \varepsilon & +1 - \varepsilon & \ldots \\
-\nicefrac12 - \varepsilon & -1 - \varepsilon & +1 - \varepsilon & +1 - \varepsilon & -1
- \varepsilon& -1 - \varepsilon & +1 - \varepsilon & +1 - \varepsilon & \ldots \\
-\nicefrac12 & +1 & +1 & -1 & -1 & +1 & +1 & -1 & \ldots \\
+\nicefrac12 & -1 & +1 & -1 & +1 & -1 & +1 & -1 & \ldots \\
+\nicefrac12 & +1 & -1 & -1 & +1 & +1 & -1 & -1 & \ldots \\
+\nicefrac12 & -1 & -1 & +1 & +1 & -1 & -1 & +1 & \ldots \\
\end{array}
\right].
\]
Now, when losses are given by $A_n^{\varepsilon,k}$, the first $k$ actions
(the good actions) perform better than the remaining $n - k$; so, for large
enough $T$, a typical algorithm will eventually recognize this and assign
the first $k$ actions equal weights (giving little or no weight to the
remaining $n - k$).
Finally, we artificially replicate each action (each row) $N/n$ times to
yield the final loss matrix $B_N^{\varepsilon,k}$ for $N$ actions:
\[
B_N^{\varepsilon,k} = \left. \left[
\begin{array}{c}
A_n^{\varepsilon,k} \\
A_n^{\varepsilon,k} \\
\vdots \\
A_n^{\varepsilon,k}
\end{array} \right] \right\}
\text{$N/n$ replicates of $A_n^{\varepsilon,k}$}.
\]

The replication of actions significantly affects the behavior of algorithms
that set parameters with respect to the number of actions $N$, which is
inflated compared to the effective number of actions $n$.
NormalHedge, having no such parameters, is completely unaffected by the
replication of actions.

We compare the performance of NormalHedge to two other representative
algorithms, which we call ``Exp'' and ``Poly''.
Exp is a time/variation-adaptive version of Hedge (exponential weights) due
to \cite{CBMS07} (roughly, $\eta_t = O(\sqrt{(\log N)/\mathrm{Var}_t})$,
where $\mathrm{Var}_t$ is the cumulative loss variance).
Poly is polynomial weights~\cite{GLS01,Gen03}, which has a parameter $p$
that is typically set as a function of the number of actions; we set $p = 2
\ln N$ as is recommended to guarantee a regret bound comparable to that of
Hedge.

Figure~\ref{fig:regret} shows the regrets to the best action versus the
replication factor $N/n$, where the effective number of actions $n$ is
held fixed. Recall that Exp and Poly have parameters set with respect to the 
number of actions $N$. 

We see from the figures that NormalHedge is
completely unaffected by the replication of actions; no matter how many
times the actions may be replicated, the performance of NormalHedge stays
exactly the same.
In contrast, increasing the replication factor affects the performance of Exp
and Poly:
Exp and Poly become more sensitive to the changes in the total losses of
the actions (\emph{e.g.}~the base of the exponent in the weights assigned
by Exp increases with $N$); so when there are multiple good actions
(\emph{i.e.}~$k > 1$), Exp and Poly are slower to stabilize their weights
over these good actions.
When $k = 1$, Exp and Poly actually perform better using the inflated value
$N$ (as opposed to $n$), as this causes the slight advantage of the single
best action to be magnified. However, this particular case is an anomaly;
this does not happen even for $k = 2$.
We note that if the parameters of Exp and Poly were set to be a function of
$n$, instead of $N$, then, then their performance would also not depend on
the replication factor (the peformance would be the same as the $N/n=1$
case).
Therefore, the degradation in performance of Exp and Poly is solely due to
the suboptimality in setting their parameters.

\begin{figure}
\vskip 0.2in
\begin{center}
\begin{tabular}{cc}
\includegraphics[width=0.5\textwidth]{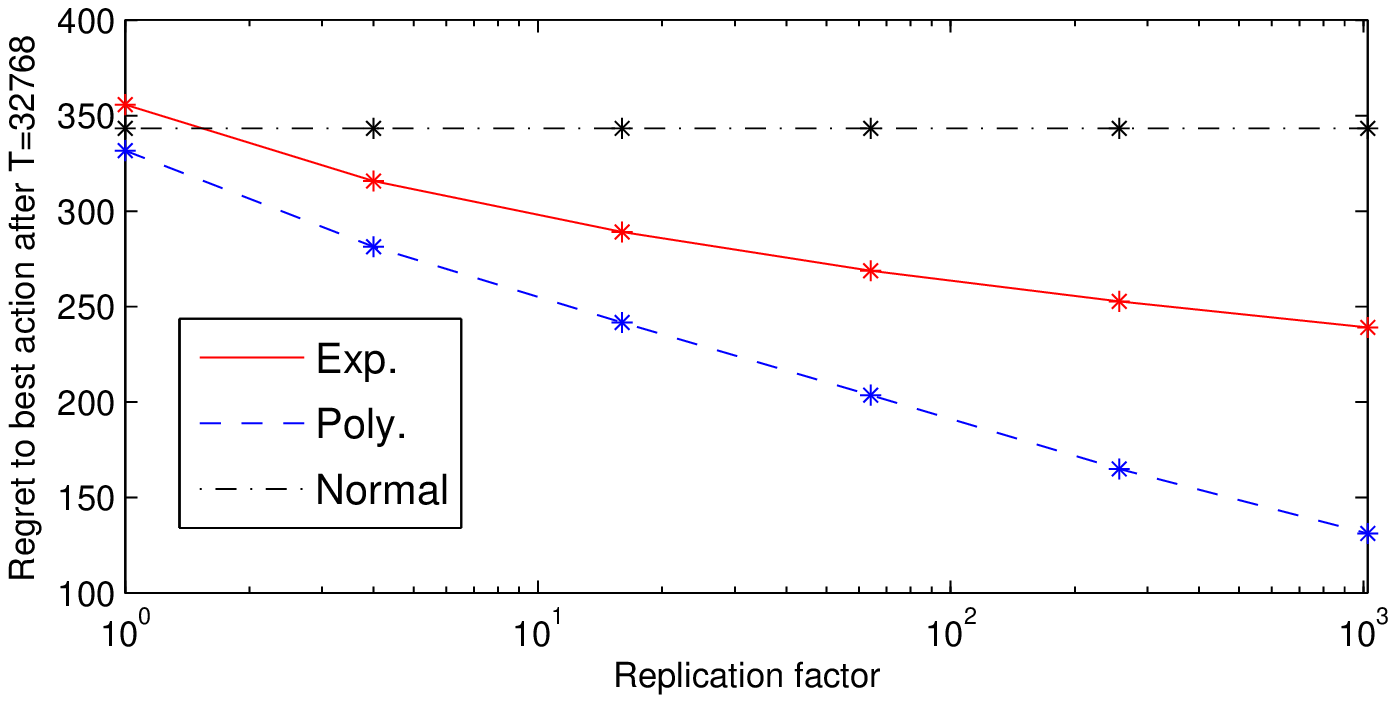} &
\includegraphics[width=0.5\textwidth]{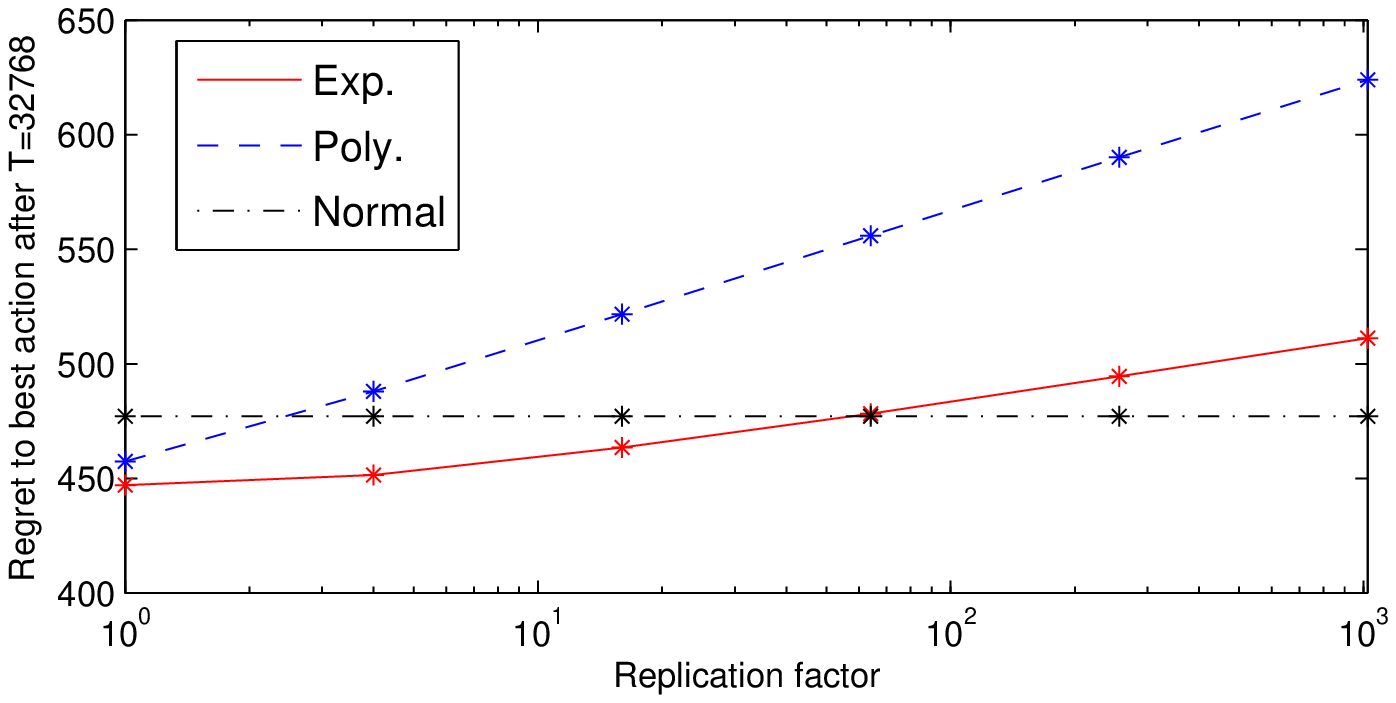} \\
$k = 1$ & $k = 2$ \\
\includegraphics[width=0.5\textwidth]{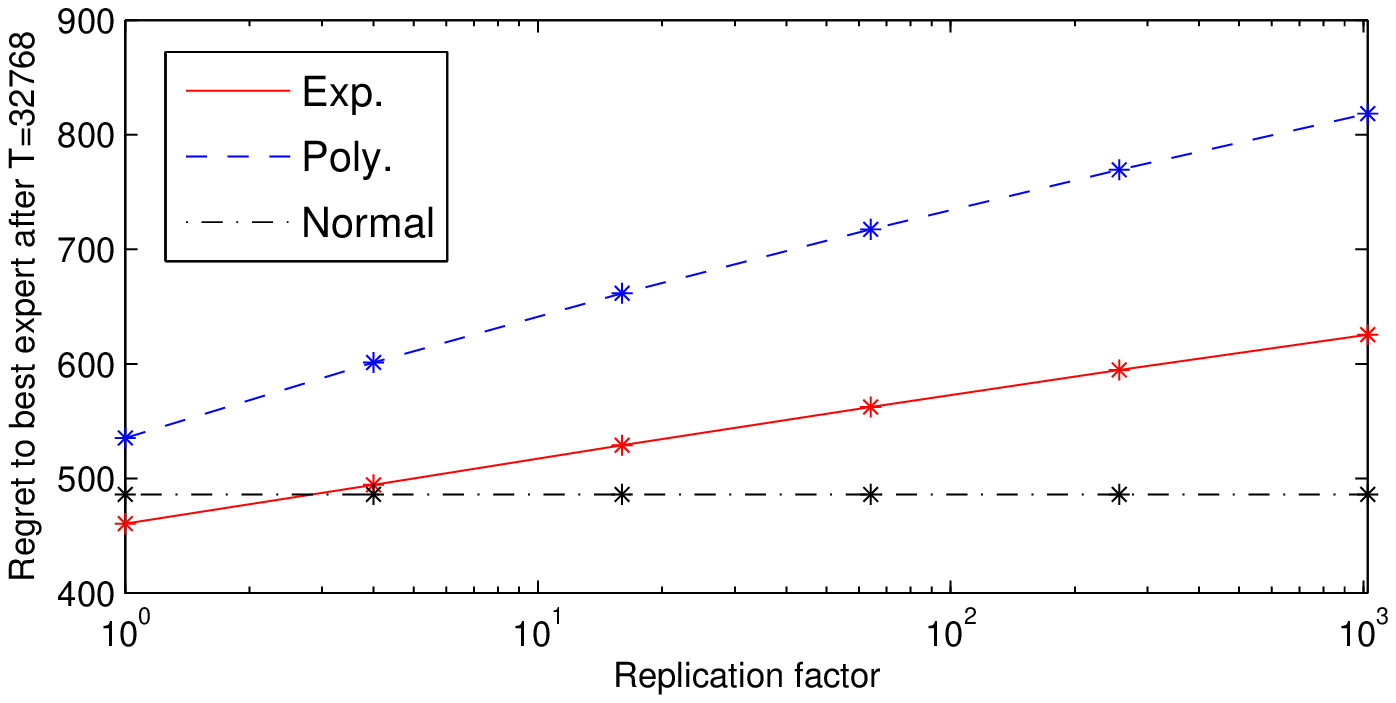} &
\includegraphics[width=0.5\textwidth]{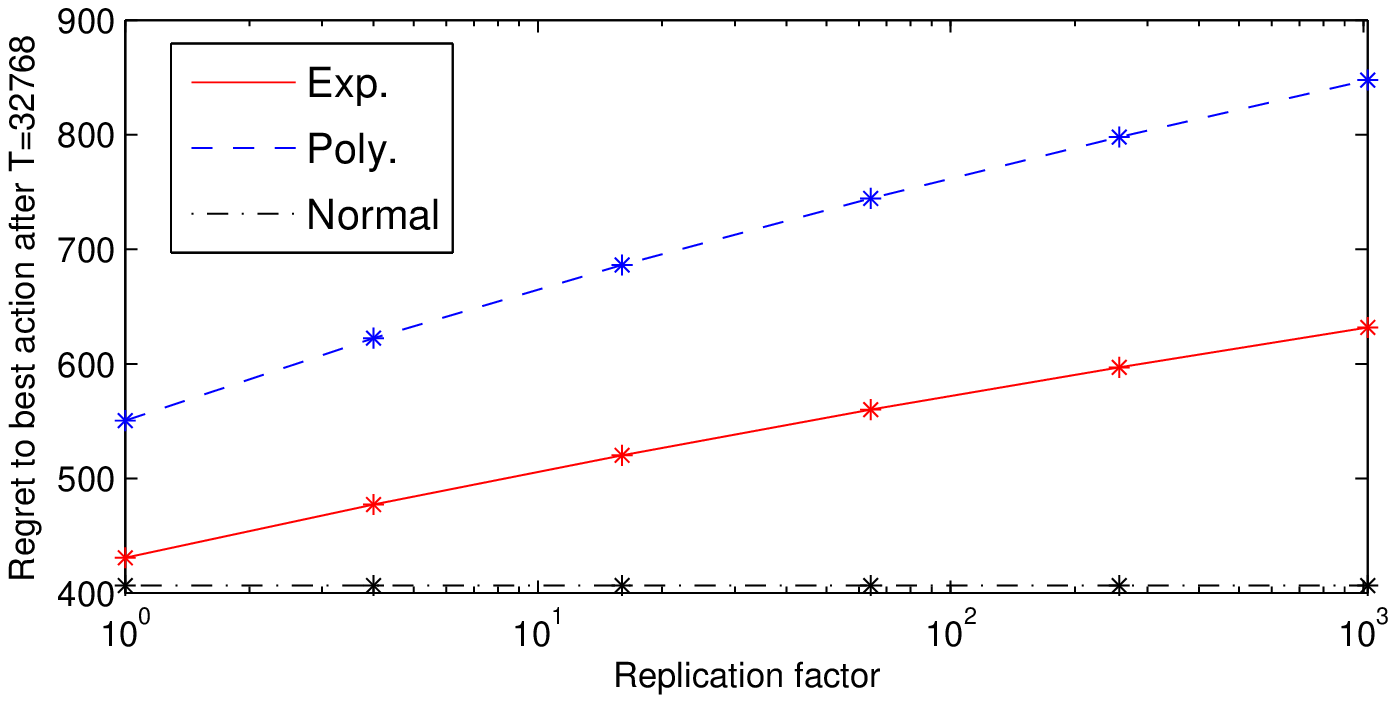} \\
$k = 8$ & $k = 32$
\end{tabular}
\caption{
Regrets to the best action after $T = 32768$ rounds, versus replication
factor $N/n$.
Recall, $k$ is the (effective) number of good actions.
Here, we fix $n=126$ and $\epsilon = 0.025$.
}
\label{fig:regret}
\end{center}
\vskip -0.2in
\end{figure} 

\section{Related work} \label{sec:related}
There has been a large amount of literature on various aspects of DTOL. The Hedge algorithm of \cite{FS97} belongs to a more general
family of algorithms, called the exponential weights algorithms; these are 
originally based
on Littlestone and Warmuth's Weighted Majority algorithm~\cite{LW94}, and
they have been well-studied.

\changeto{A significant amount of literature has also been devoted to improving the adaptivity
of these algorithms, however prior to our work, a simple, fully adaptive
algorithm for DTOL was not known. In this section, we highlight 
some of these exponential weights algorithms, as well as
a few others, as they relate to NormalHedge and parameter adaptivity.}{}

The standard measure of regret in most of these works is the regret to the
best action. The original Hedge algorithm has a regret bound of $O(\sqrt{T
\log N})$. Hedge uses a fixed learning rate $\eta$ for all iterations, and
requires \changeto{us}{one} to set $\eta$ as a function of the total number of iterations $T$. As a result, its regret bound also holds only for a
fixed $T$. \changeto{\cite{CBFH+97} provides an algorithm that}{The
algorithm of \cite{CBFH+97}} guarantees a regret
bound of $O(\sqrt{T \log N})$ to the best action\changeto{,}{} uniformly
for all $T$\changeto{,}{} by
using a doubling trick.
Time-varying learning rates for exponential weights algorithms were
considered in \cite{ACBG02}; \changeto{}{there,} they show that if $\eta_t = \sqrt{8\ln(N)/t}$, then using exponential weights with $\eta = \eta_t$ in round $t$ guarantees regret bounds of $\sqrt{2T\ln N} + O(\ln N)$ for any $T$. 
This bound provides a better regret to the best action than we do. However,
 \changeto{it}{this method} is still susceptible to poor performance, as illustrated in the example
 in Section~\ref{sec:expts}.
 Moreover, they do not consider our notion of regret.
\changeto{These algorithms can also be partly analyzed with respect to}
{Though not explicitly considered in previous works, the exponential weights
algorithms can be partly analyzed  with respect to the}
regret to the
top $\epsilon$-quantile. For any {\em fixed} $\epsilon$, Hedge
can be modified by setting $\eta$ as a function of this $\epsilon$ such that the regret to the top $\epsilon$-quantile is at most $O(\sqrt{T \log(1/\epsilon)})$. The problem with this solution is that it requires that \emph{the learning rate to be set as a function of that particular $\epsilon$} (roughly $\eta = \sqrt{(\log 1/\epsilon)/T}$). Therefore, unlike our bound, this bound does not hold uniformly for all $\epsilon$.
One way to ensure a bound for all $\epsilon$ uniformly
is to run $\log N$ copies of Hedge, each with a
learning rate set as a function of a different value of $\epsilon$. A final
master copy of the Hedge algorithm then looks at the \changeto{predictions
of}{probabilities given by} these
\changeto{multiple}{subordinate} copies\changeto{, and produces a final
prediction}{to give the final probabilities}. However, this procedure
adds an additive $O(\sqrt{T \log \log N})$ factor to the regret to the
$\epsilon$ quantile of actions, for {\em any} $\epsilon$. More importantly, this
procedure is also impractical for real applications, where one might be already working with a large set of actions. In
contrast, our solution NormalHedge is clean and simple, and we guarantee a regret bound for all values of $\epsilon$ uniformly, without any extra overhead.

More recent work in \cite{YEYS04,CBMS07,HK08} provide algorithms with significantly improved bounds when the total loss of the best action is
small, or when the total variation in the losses is small.
These bounds do not explicitly depend on $T$, and thus can often be
sharper than ones that do (including ours). We stress, however, that these
methods use a different notion of regret, and their learning
rates depend explicitly on $N$.

Besides exponential weights, another important class of online learning
algorithms are the polynomial weights algorithms studied in
\cite{GLS01,Gen03,CBL03}. These algorithms too require a parameter; this
parameter does not depend on the number of rounds $T$, but
 depends crucially on the number of actions $N$.
The weight assigned to action $i$ in round $t$ is proportional to
$([R_{i,t-1}]_+)^{p-1}$ for some $p > 1$; setting $p = 2 \ln N$ yields
regret bounds of the form $\sqrt{2eT(\ln N - 0.5})$ for any $T$.
Our algorithm and polynomial weights share the feature that zero weight is
given to actions that are performing worse than the algorithm, although the
degree of this weight sparsity is tied to the performance of the algorithm.
Finally, \cite{HP05} derive a time-adaptive variation of the
follow-the-(perturbed) leader algorithm \cite{Han57,KV05} by scaling the
perturbations by a parameter that depends on both $t$ and $N$.

\section{Analysis} \label{sec:analysis}

\subsection{Main results}
Our main result is the following theorem.

\begin{theorem}
If Normal-Hedge has access to $N$ actions, then for all loss sequences, for
all $t$, for all $0 < \epsilon \leq 1$ and for all $0 < \delta \leq 1/2$,
the regret of the algorithm to the top $\epsilon$-quantile of the actions
is at most
\[  \sqrt{ (1 + \ln(1/\epsilon)) \left( 3(1 + 50 \delta) t + \frac{16 \ln^2
N}{\delta}( \frac{10.2}{\delta^2} + \ln N)\right)}. \]
In particular, with $\epsilon=1/N$, the regret to the best action is at most
\[ \sqrt{ (1 + \ln N) \left(3(1 + 50 \delta)t + \frac{16 \ln^2
N}{\delta}(\frac{10.2}{\delta^2} + \ln N)\right)}. \]
\label{thm:main}
\end{theorem}

The value $\delta$ in Theorem~\ref{thm:main} appears to be an artifact of
our analysis; we divide the sequence of rounds into two phases -- the
length of the first is controlled by the value of $\delta$ -- and bound the
behavior of the algorithm in each phase separately.
The following corollary illustrates the performance of our algorithm for
large values of $t$, in which case the effect of this first phase (and the
$\delta$ in the bound) essentially goes away.
\begin{corollary}
If Normal-Hedge
has access to $N$ actions, then, as
$t \rightarrow \infty$, the regret of Normal-Hedge to the top
$\epsilon$-quantile of actions approaches an upper bound of
\[\sqrt{3 t(1 + \ln(1/\epsilon)) + o(t) } \, .\]
In particular, the regret of Normal-Hedge to the 
best action approaches an upper bound of
\[ \sqrt{3 t (1 + \ln N) + o(t) }\, .\] 
\end{corollary}

The proof of Theorem \ref{thm:main} follows from a combination of
Lemmas~\ref{lem:scale2regret}, \ref{lem:lowt}, and \ref{lem:scale2},
and is presented in detail at the end of the current section.

\subsection{Regret bounds from the potential equation}

The following lemma relates the performance of the algorithm at time
$t$ to the scale $c_t$.

\begin{lemma} \label{lem:scale2regret}
At any time $t$, the regret to the best action can be bounded as:
\[ \max_{i} \Rg{i}{t} \leq \sqrt{2 c_t ( \ln N + 1)}  \]
Moreover, for any $0 \leq \epsilon \leq 1$ and any $t$, the regret to the
top $\epsilon$-quantile of actions is at most
\[ \sqrt{2 c_t (\ln (1/\epsilon) + 1)}. \]
\end{lemma}

\begin{proof}
We use $\eplus$  to denote the actions that have non-zero
weight on iteration $t$. 
The first part of the lemma follows from the fact that,
for any action $i \in \eplus$,
\begin{equation*}
\exp\left( \frac{(\Rg{i}{t})^2}{2c_t} \right)
\ = \
\exp\left(\frac{([\Rg{i}{t}]_+)^2}{2 c_t}\right)
\ \leq \
\sum_{i'=1}^N \exp\left( \frac{([\Rg{i'}{t}]_+)^2}{2 c_t} \right)
\ \leq \ Ne
\end{equation*}
which implies $\Rg{i}{t} \leq \sqrt{2c_t(\ln N + 1)}$.

For the second part of the lemma, let $\Rg{i}{t}$ denote the regret of our
algorithm to the action with the $\epsilon N$-th highest regret. Then, the total potential of the actions with regrets greater than or equal to $\Rg{i}{t}$ is at least:
\[ \epsilon N \exp\left( \frac{([\Rg{i}{t}]_+)^2}{2 c_t} \right) \leq N e \] 
from which the second part of the lemma follows.
\end{proof}

\subsection{Bounds on the scale $c_t$ and proof of Theorem~\ref{thm:main}}

In Lemmas~\ref{lem:lowt} and~\ref{lem:scale2}, we bound the growth of
the scale $c_t$ as a function of the time $t$. 

The main outline of the proof of Theorem~\ref{thm:main} is as follows. As
$c_t$ increases monotonically with $t$, we can divide the rounds $t$ into two 
phases, $t < t_0$ and $t \geq t_0$, where $t_0$ is the first time such that
$$ c_{t_0} \geq \frac{4 \ln^2 N}{\delta} + \frac{16 \ln N}{\delta^3},$$ for
some fixed $\delta \in (0,1/2)$.
We then show bounds on the growth of $c_t$ for each phase separately.
Lemma~\ref{lem:lowt} shows that $c_t$ is not too large at the end of the
first phase, while Lemma~\ref{lem:scale2} bounds the per-round growth of
$c_t$ in the second phase.

\begin{lemma}
For any time $t$,
\[ c_{t+1} \leq 2 c_t(1 + \ln N) + 3\, . \]
\label{lem:lowt}
\end{lemma}

\begin{lemma}
Suppose that at some time $t_0$,  $c_{t_0} \geq \frac{4 \ln^2 N}{\delta} + \frac{16 \ln N}{\delta^3}$,
where $0 \leq \delta \leq
\frac{1}{2}$ is a constant. Then, for any
time $t \geq t_0$,
\[ c_{t+1} - c_t \leq \frac{3}{2} (1 + 49.19 \delta) \, .\]
\label{lem:scale2}
\end{lemma}

We now combine Lemmas~\ref{lem:lowt} and \ref{lem:scale2} together with
Lemma~\ref{lem:scale2regret} to prove the main theorem.
\begin{proof}[Proof of Theorem~\ref{thm:main}]
Let $t_0$ be the first time at which $c_{t_0} \geq \frac{4 \ln^2 N}{\delta}
+ \frac{16 \ln N}{\delta^3}$. Then, from Lemma \ref{lem:lowt},
\begin{eqnarray*}
c_{t_0} \leq 2 c_{t_0 - 1} (1 + \ln N) + 3,
\end{eqnarray*}
which is at most:
\begin{eqnarray*}
\frac{8 \ln^3 N}{\delta} + \frac{34 \ln^2 N}{\delta^3} + \frac{32
\ln N}{\delta^3} + 3 \leq \frac{8 \ln^3 N}{\delta} + \frac{81 \ln^2
N}{\delta^3}.
\end{eqnarray*}
The last inequality follows because $N \geq 2$ and $\delta \leq 1/2$.
By Lemma \ref{lem:scale2}, we have that for any $t \geq t_0$,
\[ c_t \leq \frac{3}{2}(1 + 49.19 \delta) (t - t_0) + c_{t_0}. \]
Combining these last two inequalities yields
\[ c_t \leq \frac{3}{2}(1 + 49.19 \delta)t + \frac{8 \ln^3 N}{\delta} + \frac{81
\ln^2 N}{\delta^3}. \]
Now the theorem follows by applying Lemma \ref{lem:scale2regret}.
\end{proof}

\section{Remaining proofs} \label{sec:proofs}

\subsection{Proof of Lemma~\ref{lem:lowt}}

\begin{proof}[Proof of Lemma \ref{lem:lowt}]
To show Lemma \ref{lem:lowt}, we first show that, for any $t$,
\begin{equation}
\frac{1}{N} \sum_i \phi(\Rg{i}{t+1}, 2c_t(1 + \ln N) + 3) \leq e.
\label{eqn:lowt}
\end{equation}
For any $i$, $\Rg{i}{t+1} \leq \Rg{i}{t} + 1$, so the left hand side of the
above inequality is at most
\begin{eqnarray*}
\frac{1}{N} \sum_i \exp\left( \frac{ (\Rg{i}{t} + 1)^2}{ 4c_t (1 +
\ln N) + 6} \right).
\end{eqnarray*}
This, in turn,  can be upper bounded by
\begin{equation} \label{eq:lowt-lhs}
\frac{1}{N} \sum_i \exp\left( \frac{\Rg{i}{t}^2}{4c_t(1 + \ln N) +
6}\right)
\cdot \exp \left( \frac{2\Rg{i}{t}}{4c_t(1 + \ln N) + 6}\right)
\cdot \exp \left( \frac{1}{4c_t(1 + \ln N) + 6}\right).
\nonumber
\end{equation}
We now bound each term in this summation. First, we note that using Lemma
\ref{lem:scale2regret}, the first term can be bounded as
\begin{eqnarray*}
\exp\left( \frac{\Rg{i}{t}^2}{4 c_t(1 + \ln N) +
6}\right) \leq \exp\left( \frac{2c_t(1 + \ln N)}{4c_t(1 + \ln N)} \right) \leq
e^{1/2}.
\end{eqnarray*}
The second term can be bounded as
\begin{eqnarray*}
\exp \left( \frac{\Rg{i}{t}}{4 c_t(1 + \ln N) + 6}\right) \leq \exp\left( \frac{ \sqrt{2c_t(1 + \ln N)}}{4c_t(1 + \ln N) + 6}
\right)
\leq
e^{\frac1{4\sqrt3}}
.
\end{eqnarray*}
The last inequality follows by noticing that
$2a + 6/a \geq 4\sqrt3$ for any $a \geq 0$, and in particular for $a =
\sqrt{2c_t(1+\ln N)}$.
Finally, the third term is trivially bounded by $e^{1/6}$.
Combining the bounds for the three terms in \eqref{eq:lowt-lhs} gives
\begin{eqnarray*}
\frac{1}{N} \sum_i \phi(\Rg{i}{t+1}, 2c_t(1 + \ln N) + 3) \leq e.
\end{eqnarray*}
Since the quantity $\sum_i \phi(\Rg{i}{t+1},c)$ is always increasing with
$c$, Equation~\eqref{eqn:lowt} implies that $c_{t+1} \leq 2c_t(1 + \ln N)
+ 3$. The lemma follows.
\end{proof}

\subsection{A bootstrap for Lemma~\ref{lem:scale2}}

Before we can prove Lemma~\ref{lem:scale2}, we first show a somewhat weaker
bound on the growth of $c_t$ with $t$ (Lemma~\ref{lem:scale}); this bound
is used in the proof of Lemma~\ref{lem:scale2} which concludes with the
tighter bound on $c_{t+1} - c_{t}$.

\begin{lemma}
Suppose that at some time $t_0$, $c_{t_0} \geq \frac{16 \ln N}{\delta^2}$,
where $0 \leq \delta \leq
1/2$ is a constant.
Then, for any time $t \geq t_0$, 
\[ c_{t+1} - c_t \leq \frac{e^{\delta}\left(\frac{3}{2} +
\delta + \ln
N\right)}{ 1 - \delta^2 e^\delta}. \]
\label{lem:scale}
\end{lemma}

The main idea behind the proof of Lemma~\ref{lem:scale} is as follows.
First, we use Lemma~\ref{lem:cdiff} to show that $c_t$ is monotonic in $t$,
and to get an expression for $c_{t+1} - c_{t}$ as a ratio of some
derivatives and double derivatives of the potential function $\phi$.
Next, we use Lemma~\ref{lem:upperbound} and Corollary~\ref{cor:lowerbound}
to bound the numerator and denominator of this ratio. Combining these
bounds gives us a proof of Lemma~\ref{lem:scale}.

We denote by $\epp \doteq \eplus \cup \eplustwo$ the actions relevant to
the change of potential between iterations $t$ and $t+1$ (recall, $\eplus$
are the actions with non-zero weight on iteration $t$).
\begin{lemma}
At any time $t$, 
\[ c_{t+1} - c_t \geq 0. \]
Moreover, $c_{t+1} - c_{t}$ is at most:
\begin{equation*}
\frac{ \sum_{i \in \epp} \frac{(\rg{i}{t+1})^2}2
\left(
\frac{1}{c_t} + \frac{\rhog{i}{t}^2}{c_t^2} \right)
\exp\left(\frac{\rhog{i}{t}^2}{2c_t}\right) }
{ \sum_{i \in \eplustwo}
\frac{(\Rg{i}{t+1})^2}{\eps_{t+1}^2} \exp\left(\frac{(\Rg{i}{t+1})^2}{2 \eps_{t+1}}\right)} 
\label{eqn:cdiff}
\end{equation*}
where $\eps_{t+1}$ lies in between $c_t$ and $c_{t+1}$ and for each
$i$, $\rhog{i}{t}$ lies between $\Rg{i}{t}$ and $\Rg{i}{t+1}$.
\label{lem:cdiff}
\end{lemma}

\begin{proof}
We consider the change in average potential due to the regrets changing
from $\Rg{i}{t}$ to $\Rg{i}{t+1}$ (with the scale fixed at $c_t$), and then
the change due to the scale changing from $c_t$ to $c_{t+1}$:
\begin{eqnarray}
0
& = & \sum_{i=1}^N \phi(\Rg{i}{t+1},c_{t+1})
- \sum_{i=1}^N \phi(\Rg{i}{t},c_t)
\nonumber \\
& = & \sum_{i=1}^N
\phi(\Rg{i}{t+1},c_t) - \phi(\Rg{i}{t},c_t)
\label{eq:regret-change} \\
& & + \sum_{i=1}^N
\phi(\Rg{i}{t+1},c_{t+1}) - \phi(\Rg{i}{t+1},c_t).
\label{eq:scale-change}
\end{eqnarray}
It is clear that the sum in \eqref{eq:regret-change} can be restricted to
$i \in \epp$, and that the sum in \eqref{eq:scale-change}
can be restricted to $i \in \eplustwo$.
We now will express \eqref{eq:scale-change} in terms of $c_{t+1} - c_t$ and
employ upper and lower bounds on \eqref{eq:regret-change}.

First, we derive bounds on \eqref{eq:regret-change}.
Let $\psi(x) = \exp(x^2/(2c_t))$.
Then $f(x) = \phi(x,c_t)$ can be written as
\begin{equation}
f(x) \ = \
\left\{ \begin{array}{ll}
\psi(x) & \text{if} \ x \geq 0 \\
\psi(0) & \text{if} \ x < 0.
\end{array} \right.
\nonumber
\end{equation}
The function $f$ satisfies the preconditions of Lemma~\ref{lem:taylor}
(deferred to the end of the section), so
we have
\begin{equation}
f(\Rg{i}{t+1}) - f(\Rg{i}{t})
\ \geq \
f'(\Rg{i}{t}) \rg{i}{t+1}
\label{eq:numerator-lb}
\end{equation}
and
\begin{equation}
f(\Rg{i}{t+1}) - f(\Rg{i}{t})
\ \leq \ f'(\Rg{i}{t}) \rg{i}{t+1} +
\frac{\psi''(\rhog{i}{t})}{2} \rg{i}{t+1}^2
\label{eq:numerator-ub}
\end{equation}
where $\min\{\Rg{i}{t},\Rg{i}{t+1}\} \leq \rhog{i}{t} \leq
\max\{\Rg{i}{t},\Rg{i}{t+1}\}$ and $\rg{i}{t+1} = \Rg{i}{t+1} - \Rg{i}{t}$.
Now we sum both the lower and upper bounds (Eqs.~\eqref{eq:numerator-lb} and
\eqref{eq:numerator-ub}) over $i \in \epp$ and apply the fact
\begin{equation}
\sum_{i \in \epp} f'(\Rg{i}{t}) \rg{i}{t+1}
\ = \ \sum_{i=1}^N f'(\Rg{i}{t}) \rg{i}{t+1}
\ = \ 0
\nonumber
\end{equation}
which follows easily from the fact that the weight assigned to an action
$i$ in trial $t+1$ is proportional to $f'(\Rg{i}{t})$.
Thus,
\begin{eqnarray}
0 & \leq &
\sum_{i \in \epp}
\phi(\Rg{i}{t+1},c_t) - \phi(\Rg{i}{t},c_t)
\label{eq:regret-change-lb}
\\
& \leq &
\sum_{i \in \epp} \left(
\frac{1}{c_t} + \frac{\rhog{i}{t}^2}{c_t^2} \right)
\exp\left(\frac{\rhog{i}{t}^2}{2c_t}\right)
\cdot (\rg{i}{t+1})^2.
\label{eq:regret-change-ub}
\end{eqnarray}

To deal with \eqref{eq:scale-change}, we view it as a function of $c_{t+1}$
and then equated via Taylor's theorem to a first-order expansion around
$c_t$
\begin{equation*}
- (c_{t+1} - c_t) \cdot \sum_{i \in \eplustwo}
\frac{(\Rg{i}{t+1})^2}{2\eps_{t+1}^2}
\exp\left(\frac{(\Rg{i}{t+1})^2}{2\eps_{t+1}}\right)
\end{equation*}
for some $\eps_{t+1}$ between $c_t$ and $c_{t+1}$.
Substituting this back into \eqref{eq:scale-change}, we have
\begin{equation}
\sum_{i \in \epp}
\phi(\Rg{i}{t+1},c_t) - \phi(\Rg{i}{t},c_t)
\ = \ (c_{t+1} - c_t) \cdot \sum_{i \in \eplustwo}
\frac{(\Rg{i}{t+1})^2}{2\eps_{t+1}^2}
\exp\left(\frac{(\Rg{i}{t+1})^2}{2\eps_{t+1}}\right)
.
\label{eq:scale-change-taylor}
\end{equation}
The summation on the right-hand side is non-negative,
as is the summation on the left-hand side (recall the lower bound
\eqref{eq:regret-change-lb}), so $c_{t+1} - c_t$ is non-negative as well.
This shows the first part of the lemma.
The second part follows from re-arranging
Eq.~\eqref{eq:scale-change-taylor} and applying the upper bound
\eqref{eq:regret-change-ub}.
\end{proof}

\begin{lemma}
Let, for some $t=t_0$, $c_{t_0} \geq \frac{16\ln N}{\delta^2} +
\frac{1}{\delta}$, for some $0 \leq \delta \leq 1$. Then, for any $t \geq
t_0$,
\begin{equation*}
\sum_{i \in \epp} \exp\left(\frac{\rhog{i}{t}^2}{2 c_t}\right) \leq
e^{\delta} N e
\end{equation*}
and also
\begin{equation*}
\sum_{i \in \epp} \frac{\rhog{i}{t}^2}{2c_t}
\exp\left(\frac{\rhog{i}{t}^2}{2 c_t}\right)
\ \leq \ e^{\delta}
(\delta + 1 + \ln N) N e
\end{equation*}
where the $\rhog{i}{t}$ are the values introduced in Lemma~\ref{lem:cdiff}.
\label{lem:upperbound}
\end{lemma}

\begin{proof}
Pick any $i \in \epp$.
If $\Rg{i}{t} \geq 0$, then $|\rhog{i}{t}| \leq \Rg{i}{t} 
+ 1 = [\Rg{i}{t}]_+ + 1$.
Otherwise $\Rg{i}{t+1} \geq 0 > \Rg{i}{t}$.
But since $|\rg{i}{t+1}| = |\Rg{i}{t+1} - \Rg{i}{t}| \leq 1$, it must be
that $|\rhog{i}{t}| \leq 1 = [\Rg{i}{t}]_+ + 1$.
Therefore
\begin{equation*}
\rhog{i}{t}^2 \ \leq \ ([\Rg{i}{t}]_+ + 1)^2,
\end{equation*}
which in turn implies
\begin{equation*}
\exp\left(\frac{\rhog{i}{t}^2}{2 c_t}\right)
\ \leq \
\exp\left(\frac{([\Rg{i}{t}]_+ + 1)^2}{2 c_t}\right)
\ = \
\exp\left(\frac{([\Rg{i}{t}]_+)^2}{2c_t}\right)
\cdot \exp\left(\frac{2[\Rg{i}{t}]_+}{2c_t}\right)
\cdot \exp\left(\frac{1}{2c_t}\right).
\end{equation*}
To prove the first claim, it suffices to show that each
of the two exponentials in the final product is bounded by $e^{\delta/2}$.
Since $c_t \geq c_{t_0} = (16 \ln N)/\delta^2 + 1/\delta$, we have
$\exp(1/(2c_t)) \leq e^{\delta/2}$.
Also, Lemma \ref{lem:scale2regret} imply
\begin{equation*}
\exp\left(\frac{[\Rg{i}{t}]_+}{c_t}\right)
\leq \exp\left( \frac{\sqrt{2 (1 + \ln N)}}{\sqrt{c_t}} \right)
\ \leq \
\exp\left( \frac{2 \delta \sqrt{\ln N}}{ 4 \sqrt{\ln N} } \right) \leq e^{\delta/2},
\end{equation*}
so the first claim follows.

To prove the second claim, we use the first claim to derive the fact
\begin{equation*}
\max_{i' \in \epp}
\frac{\rhog{i'}{t}^2}{2c_t}
\ \leq \
\ln \sum_{i' \in \epp}
\exp\left(\frac{\rhog{i'}{t}^2}{2c_t}\right)
\ \leq \ \delta + 1 + \ln N
\end{equation*}
which in turn is combined again with the first claim to arrive at
\begin{equation*}
\sum_{i \in\epp} \frac{\rhog{i}{t}^2}{2c_t}
\exp\left(\frac{\rhog{i}{t}^2}{2 c_t}\right) \ \leq \
\left( \max_{i' \in \epp}
\frac{\rhog{i'}{t}^2}{2c_t} \right)
\sum_{i \in\epp}
\exp\left(\frac{\rhog{i}{t}^2}{2 c_t}\right) \ \leq \
(\delta + 1 + \ln N) e^\delta N e,
\end{equation*}
completing the proof.
\end{proof}

\begin{lemma}
Let $B \geq 1$.
If $\sum_{i=1}^N e^{x_i} \geq BN$ for some $x \geq 0$,
then $\sum_{i=1}^N x_i e^{x_i} \geq BN \ln B$.
\label{lem:lowerbound}
\end{lemma}
\begin{proof}
We consider minimizing $f(x) = \sum_{i=1}^N x_i e^{x_i}$ under the constraint
$\sum_{i=1}^N e^{x_i} \geq BN$.
Define the Lagrangian function
$L(x,\lambda) = \sum_{i=1}^N x_i e^{x_i} +
\lambda (BN - \sum_{i=1}^N e^{x_i})$.
Then $(\partial/\partial x_i) L(x,\lambda) = (x_i + 1 - \lambda) e^{x_i}$,
which is $0$ when $x_i = \lambda - 1$.
Let $g(\lambda) = L(x^*, \lambda)$ be the dual function, where $x_i^* =
\lambda - 1$.
Then $g$ is maximized when $\lambda = 1 + \ln B$.
By weak duality, $\sup_\lambda g(\lambda) \leq \inf_x f(x)$ (with the
constraints on $x$), so $f(x) \geq g(1+\ln B) = BN \ln B$.
\end{proof}
The lemma above leads to the following corollary.
\begin{corollary}
\label{cor:lowerbound}
For any $t$,
\[ \sum_{i \in \eplustwo} \frac{(\Rg{i}{t+1})^2}{2 c_{t+1}}
\exp\left( \frac{(\Rg{i}{t+1})^2}{2 c_{t+1}} \right) \geq Ne. \]
\end{corollary}
\begin{proof}
Let $x_i = ([\Rg{i}{t+1}]_+)^2 / (2c_{t+1})$, so we have $\sum_{i=1}^N
\exp(x_i) = Ne$ by Equation~\eqref{eq:potential}.
Now Lemma \ref{lem:lowerbound} implies $\sum_{i=1}^N x_i \exp(x_i) \geq
Ne$.
The corollary follows since $[\Rg{i}{t+1}]_+ = 0$ whenever $i \not\in
E_{t+1}$.
\end{proof}

Now we prove Lemma~\ref{lem:scale}.
\begin{proof}[Proof of Lemma \ref{lem:scale}]
By Lemma~\ref{lem:cdiff} and the fact $|\rg{i}{t+1}| \leq 1$, we have
\begin{equation*}
0 \ \leq \ c_{t+1} - c_t
\ \leq \ \frac{ \sum_{i
\in \epp} (\frac{1}{2 c_t} + \frac{\rhog{i}{t}^2}{2 c_t^2}) \exp\left( \frac{\rhog{i}{t}^2}{2 c_t} \right) } {
\sum_{i \in \eplustwo} \frac{(\Rg{i}{t+1})^2}{2 c_{t+1}^2} \exp\left(
\frac{ (\Rg{i}{t+1})^2}{2 c_{t+1}} \right)}.
\end{equation*}
Combining this with Lemma~\ref{lem:upperbound}, Corollary~\ref{cor:lowerbound}, we have
\begin{equation*}
c_{t+1} - c_{t} \ \leq \
\frac{c_{t+1}}{c_t} \cdot \frac{ e^{\delta} \left(\frac{1}{2} + \delta + 1 + \ln
N\right) N e}{N e} \ = \ \left(1 + \frac{c_{t+1} - c_{t}}{c_t}\right) \cdot
e^{\delta} \left(1 + \frac{1}{2} + \delta + \ln N\right).
\end{equation*}
Re-arranging, and using the fact that $c_t \geq c_{t_0} \geq (16 \ln
N)/\delta^2$, we get that, 
\begin{equation*}
c_{t+1} - c_{t}
\ \leq \ \frac{ e^{\delta}\left(1 + \frac{1}{2} + \delta
+ \ln N\right)}{1 - \frac{e^{\delta} \left( 1 + \frac{1}{2} +
\delta + \ln N \right)}{c_{t_0}}}
\ \leq \ \frac{ e^{\delta}(\frac{3}{2} + \delta + \ln N)}{1 - \delta^2 e^{\delta}(\frac{1}{16 \ln N} +
\frac{\frac{1}{2} +
\delta}{16 \ln N} +
\frac{1}{16})}.
\end{equation*}
The rest of the lemma follows by plugging in the fact that $N \geq 2$, and
$\delta \leq 1/2$.
\end{proof}

\subsection{Proof of Lemma~\ref{lem:scale2}}

Finally, we are ready to prove Lemma~\ref{lem:scale2}. As in the proof of
Lemma~\ref{lem:scale}, here too, we start with an upper bound on $c_{t+1} -
c_{t}$, obtained from Lemma~\ref{lem:cdiff}. We then use this upper bound,
and the bound in Lemma~\ref{lem:scale} to get a finer bound on the quantity
$c_{t+1} - c_{t}$.

\begin{proof}[Proof of Lemma \ref{lem:scale2}]
We divide the actions into two sets:
\begin{eqnarray*}
\sone & = & \{ i \in \epp : [\Rg{i}{t+1} ]_+ + 1 \leq
\sqrt{2 c_t \delta} \} \\
\stwo & = & \{ i \in \epp : [\Rg{i}{t+1}]_+ + 1 >
\sqrt{2 c_t \delta} \}.
\end{eqnarray*}
Using the fact that $|\rg{i}{t+1}| \leq 1$,
the bound from Lemma~\ref{lem:cdiff} can be written as
\begin{align*}
c_{t+1} - c_{t}
& \leq \frac{ \frac{1}{2c_t} \sum_{i \in \epp}
\exp\left( \frac{\rhog{i}{t}^2}{2 c_t} \right)}{\sum_{i
\in \eplustwo} \frac{(\Rg{i}{t+1})^2}{2 c_{t+1}^2} \exp\left( \frac{
(\Rg{i}{t+1})^2}{2 c_{t+1}} \right)} \\
& \quad {} + 
\frac{ \sum_{i \in \sone} \frac{\rhog{i}{t}^2}{2c_t^2} \exp\left( \frac{\rhog{i}{t}^2}{2c_t^2} \right)
}{\sum_{i \in \eplustwo} \frac{\Rg{i}{t+1}^2}{2 c_{t+1}^2}
\exp\left(\frac{([\Rg{i}{t+1}]_+)^2}{2 c_{t+1}^2}\right)} \\
& \quad {} + \frac{\sum_{i \in \stwo}\frac{\rhog{i}{t}^2}{2c_t^2} \exp\left(
\frac{\rhog{i}{t}^2}{2c_t} \right)  }{\sum_{i \in
\eplustwo}\frac{(\Rg{i}{t+1})^2}{2 c_{t+1}^2}
\exp\left(\frac{(\Rg{i}{t+1})^2}{2 c_{t+1}}\right) }.
\end{align*}
We will upper-bound each of these three terms separately.

The first term is bounded by $(c_{t+1}/c_t) e^\delta/2$ using
Lemma~\ref{lem:upperbound} and Corollary \ref{cor:lowerbound}.

To bound the second term, the definition of $\sone$ implies
\begin{equation*}
\sum_{i \in \sone} \frac{([\Rg{i}{t+1}]_+ + 1)^2}{2 c_t}
\exp\left(\frac{([\Rg{i}{t+1}]_+ + 1)^2}{2 c_t} \right) \\
\ \leq \ N \frac{2 c_t \delta}{2 c_t} \exp\left(\frac{2 c_t \delta}{2
c_t}\right) \leq \delta e^{\delta} N.
\end{equation*}
Now Corollary \ref{cor:lowerbound} implies a bound of
$(c_{t+1}/c_t) (\delta e^\delta / e)$.

Now we bound the third term.
Note that since $\sqrt{2c_t\delta} \geq \sqrt{2c_{t_0}\delta} > 1$, we have
that each $i \in \stwo$ is also in $\eplustwo$.
So the third term is bounded above by the largest ratio
\[ \frac{ \frac{\rhog{i}{t}^2}{2 c_t^2} \exp\left( \frac{\rhog{i}{t}^2}{2 c_t} \right)  } { \frac{(\Rg{i}{t+1})^2}{2 c_{t+1}^2}
\exp\left( \frac{(\Rg{i}{t+1})^2}{2 c_{t+1}}\right) } \]
over all $i \in \eplustwo$.
Since $|\rhog{i}{t}| \leq \Rg{i}{t+1} + 1$, each such ratio
is at most
\begin{equation*}
\frac{c_{t+1}^2}{c_t^2} \cdot \left( \frac{\Rg{i}{t+1} + 1}{
\Rg{i}{t+1}} \right)^2
\cdot \exp\left(\frac{1}{2c_t}
+ \frac{\Rg{i}{t+1}}{c_t}\right)
 \cdot \exp\left( \frac{
(\Rg{i}{t+1})^2 (c_{t+1} - c_{t}) }{2 c_t c_{t+1}} \right).
\end{equation*}
We bound each factor in this product
(deferring $c_{t+1}^2/c_t^2$ until later).
\begin{itemize}
\item As $\Rg{i}{t+1} + 1 \geq \sqrt{2 c_t \delta}$ and $c_t \geq c_{t_0} \geq
10/\delta^3$, we have $\sqrt{2 c_t \delta} \geq 1/\delta$ and
\[ \left( \frac{\Rg{i}{t+1}  + 1}{\Rg{i}{t+1}}\right)^2
\leq \frac{1}{(1 - \delta)^2}. \]

\item By Lemma~\ref{lem:scale2regret}, we have
$$ \Rg{i}{t+1}\leq \Rg{i}{t} +1 \leq \sqrt{2c_t(1 + \ln N)} +
1, $$ so
\begin{align*}
\exp\left(\frac1{2c_t} + \frac{\Rg{i}{t+1} }{c_t}\right)
& \leq
\exp\left(\frac{3}{2c_t} + \frac{\Rg{i}{t} }{c_t}\right) \\
& \leq
\exp\left(\frac{3}{2c_t} + \sqrt{\frac{2(1+\ln N)}{c_t}}\right)
\leq e^{3\delta^2/20 + 3\delta/5}
\end{align*}
since $c_t \geq c_{t_0} \geq (16 \ln N)/\delta^2 \geq 10 / \delta^2$.

\item We use Lemma \ref{lem:scale} with $c_{t_0} \geq (16 \ln
N)/\delta^2$, and $\delta \leq 1/2$ to obtain the crude bound
\begin{equation}
c_{t+1} - c_{t} \leq e^{\delta} (4 + 2 \ln N).
\label{eqn:deltabound}
\end{equation}
Now using this bound along with Lemma~\ref{lem:scale2regret}
and $\delta \leq 1/2$ gives
\begin{equation*}
\frac{ (\Rg{i}{t+1})^2 (c_{t+1} - c_{t}) }{2 c_t
c_{t+1}}
\ \leq \ \frac{ e^{\delta} (4 + 2 \ln N)(1 + \ln N)}{ c_t }
\ \leq \ \frac{6.2 + 4.7 \ln N + 3.1 \ln^2 N}{c_{t_0}}
\end{equation*}
which is at most $\delta$ since
$c_{t_0} \geq (16 \ln N) / \delta^2 + (4 \ln^2 N)/\delta$.
\end{itemize}
Therefore, the third term is bounded by
\[ \frac{c_{t+1}^2}{c_t^2} \cdot \frac{\exp(1.6\delta + 0.15\delta^2)}{(1 -
\delta)^2}
\leq
\frac{c_{t+1}^2}{c_t^2} \cdot \frac{e^{2\delta}}{(1 - \delta)^2}. \]

Collecting the three terms in the bound for $c_{t+1} - c_t$,
\[ c_{t+1} - c_{t} \leq \frac{c_{t+1}}{c_t} \cdot
\frac{e^{\delta}}{2} + \frac{c_{t+1}}{c_t} \cdot \frac{\delta e^{\delta}}{e}
+ \frac{c_{t+1}^2}{c_t^2} \cdot \frac{e^{2\delta}}{(1 -
\delta)^2}. \]
We bound the ratio $c_{t+1}/c_t$ as
\begin{equation*}
\frac{c_{t+1}}{c_t}
\ \leq \ 1 + \frac{c_{t+1} - c_t}{c_t} \leq 1 +
\frac{e^{\delta}(4 + 2 \ln N)}{c_t}
\ \leq \ 1 + \frac{4\delta^2}{10} + \frac{\delta^2}{8}
\ = \ 1 + \frac{21\delta e^{\delta}}{40}
\ \leq \ 1 + \delta
\end{equation*}
where we have used the bound in \eqref{eqn:deltabound},
$c_{t_0} \geq (16 \ln N)/\delta^2$, and $\delta \leq 1/2$.
Therefore, we have
\[ c_{t+1} - c_{t} \leq
\frac12 \cdot e^{\delta}(1+\delta)
+ 1 \cdot \frac{e^{2\delta}(1+\delta)^2}{(1-\delta)^2}
+ \frac{\delta e^\delta(1+\delta)}{e} \]
To finish the proof, we use the fact that $\delta \leq \frac{1}{2}$. Using
this condition, and a Taylor series expansion, when $0 \leq \delta \leq
\frac{1}{2}$, $e^{\delta} \leq 1 + \sqrt{e} \delta \leq 1 + 1.65 \delta$.
Using this fact,
\[ \frac{1}{2} e^{\delta}(1 + \delta) + \frac{\delta}{e} e^{\delta}(1 +
\delta)\] is at most 
\[ \frac{1}{2} + 3.02 \delta + 2.63 \delta^2 + 0.61 \delta^3 \]
which in turn is at most $0.5 + 3.49 \delta$.
Moreover, 
\[ \frac{e^{2 \delta}(1 + \delta)^2}{(1 - \delta)^2} \leq e^{2 \delta} (1 +
4 \delta)^2 \]
which again is at most
\[ (1 + 3.3 \delta) (1 + 4 \delta)^2 \]
Using the fact that $\delta \leq \frac{1}{2}$, this is at most $1 + 45.7
\delta$. The lemma follows by combining this with the bound in the previous
paragraph.
\end{proof}

\begin{lemma} \label{lem:taylor}
Let $\psi:\R \to \R$ be any continuous, twice-differentiable, convex
function such that for some $a \in \R$, $\psi$ is non-decreasing on
$[a,\infty)$ and $\psi'(a) = 0$.
Define $f:\R \to \R$ by
\begin{equation*}
f(x)
\ = \
\left\{\begin{array}{ll}
\psi(x) & \text{if} \ x \geq a \\
\psi(a) & \text{if} \ x < a,
\end{array}\right.
\end{equation*}
Then for any $x_0, x \in \R$,
\begin{equation*}
f'(x_0) (x - x_0)
\ \leq \
f(x) - f(x_0)
\ \leq \
f'(x_0) (x - x_0) + \frac{\psi''(\xi)}{2} (x - x_0)^2
\end{equation*}
for some $\min\{x_0,x\} \leq \xi \leq \max\{x_0,x\}$.
\end{lemma}
\begin{proof}
The lower bound follows from the convexity of $f$, which is inherited from
the convexity of $\psi$.
For the upper bound, we first consider the case $x_0 < a$ and $x \geq a$.
Then, for some $\xi \in [a,x]$,
\begin{eqnarray}
f(x) - f(x_0)
& = & \psi(x) - \psi(a) \nonumber \\
& = & \psi'(a) (x - a) + \frac{\psi''(\xi)}{2} (x - a)^2
\label{eq:taylor-eq1} \\
& \leq & \psi'(a) (x - x_0) + \frac{\psi''(\xi)}{2} (x - x_0)^2
\label{eq:taylor-eq2} \\
& = & f'(x_0) (x - x_0) + \frac{\psi''(\xi)}{2} (x - x_0)^2
\nonumber
\end{eqnarray}
where \eqref{eq:taylor-eq1} follows by Taylor's theorem
and \eqref{eq:taylor-eq2} follows since $x_0 \leq a < x$, $\psi'(a) \geq
0$, and $\psi''(\xi) \geq 0$.
The case $x_0 \geq a$ and $x < a$ is analogous, and the remaining cases are
immediate using Taylor's theorem.
\end{proof}

\begin{figure}[h]
\begin{center}
\framebox[0.95\textwidth]{\begin{minipage}{0.9\textwidth}
\begin{center}
\begin{eqnarray*}
\phi(x,c)
& = & \exp\left(\frac{([x]_+)^2}{2c}\right) \\
\frac{\partial}{\partial x} \phi(x,c)
& = & \frac{[x]_+}{c} \exp\left(\frac{([x]_+)^2}{2c}\right) \\
\frac{\partial^2}{\partial x^2} \phi(x,c)
& = & \left\{
\begin{array}{cc}
\left( \frac{1}{c} + \frac{x^2}{c^2} \right)
\exp\left(\frac{x^2}{2c}\right) & \text{if $x>0$} \\
0 & \text{if $x<0$}
\end{array} \right. \\
\frac{\partial}{\partial c} \phi(x,c)
& = & -\frac{([x]_+)^2}{2c^2} \exp\left(\frac{([x]_+)^2}{2c}\right)
\end{eqnarray*}
\end{center}
\end{minipage}}
\caption{The potential function $\phi(x,c)$ and its derivatives.}
\label{fig:phi}
\end{center}
\end{figure}

\subsubsection*{References}
{\def\section*#1{}\small \bibliography{paper} \bibliographystyle{unsrt}}

\end{document}